\documentclass[runningheads,a4paper,10pt]{llncs}
\usepackage[utf8]{inputenc}

\usepackage[colorlinks=true]{hyperref}
\usepackage{graphicx}
\usepackage{caption}
\usepackage{algorithmic}
\usepackage{algorithm}
\usepackage{subcaption}
\usepackage{amsmath}
\usepackage{cite}

\DeclareMathOperator*{\argmax}{argmax}

\begin{document}

\mainmatter
\title{Max-Sum Diversification, Monotone Submodular Functions and Semi-metric Spaces}

\titlerunning{Max-Sum Diversification and Semi-metric Spaces}

\author{Sepehr~Abbasi~Zadeh and
        Mehrdad~Ghadiri
        \\\{sabbasizadeh, ghadiri\}@ce.sharif.edu}
\institute{School of Computer Engineering, Sharif University of Technology, Iran}
        
\authorrunning{S. Abbasi Zadeh and M. Ghadiri}
\maketitle

\begin{abstract}
In many applications such as web-based search, document summarization, facility location and other applications, the results are preferable to be both representative and diversified subsets of documents.
The goal of this study is to select a good ``quality'', bounded-size subset of a given set of items, while maintaining their diversity relative to a semi-metric distance function.
This problem was first studied by Borodin et al\cite{borodin}, but a crucial property used throughout their proof is the triangle inequality. In this modified proof we want to relax the triangle inequality and relate the approximation ratio of max-sum diversification problem to the parameter of the relaxed triangle inequality in the normal form of the problem (i.e., a uniform matroid) and also in an arbitrary matroid.
\end{abstract}

\section*{Introduction}
In many search applications, the search engine should guess the correct results from a given query; therefore, it is important to deliver a diversified and representative set of documents to a user. Diversification can be viewed as a trade-off between having more relevant results and having more diverse results among the top results for a given query\cite{chen2006less}. ``Jaguar'' is a cliche example in the diversification literature \cite{wang2007learn, clarke2008novelty, chekuri1997web}, but it illustrates the point perfectly as it has different meanings including car, animal, and a football team. 
A set of good ``quality'' result should cover all these diversified items. 
The paper by Borodin et al\cite{borodin} determines the good quality results with a monotone submodular function and defines diversity as the sum of distances between selected objects. 
Since they consider the distances to be metric, they ask in the conclusion section:
\begin{quote}
For a relaxed version of the triangle inequality can we relate the approximation ratio to the parameter of a relaxed triangle inequality?
\end{quote}

In this study we answer to this question. We call this relaxed triangle inequality distance as semi-metric. A semi-metric distance on a set of items is just like a metric distance, but the triangle inequality is relaxed with a parameter $\alpha \geq 1$ (i.e., $d(u,v) \leq \alpha(d(v,w) + d(w,u))$).
Answering to this question will make this method applicable to algorithms that are defined on semi-metric spaces, e.g., \cite{veltkamp2001shape, fagin2003comparing, o2002streaming}.
The IBM's Query by Image Content system is one of the other best-known examples of the semi-metric usage in practice; although, it does not satisfy the triangle inequality\cite{fagin1998relaxing}.
By modifying the analysis of the previous proposed algorithms in \cite{borodin}, we will show that these algorithms can still achieve a $2\alpha$-approximation for this question in the case that there is not any matroid constraint and a 2$\alpha^2$-approximation for an arbitrary matroid constraint. In other words, these new modified analysis are a generalization of the previous analysis as they are consistent with the previous approximation ratios for $\alpha=1$ (i.e., the metric distance). 
\section*{Problem 1. Max-Sum Diversification}
Let $U$ be the underlying ground set, and let $d(.,.)$ be a semi-metric distance function on $U$.
The goal of the problem is to find a subset $S \subseteq U$ that:\\\\
\hspace*{2cm} maximizes $f(S)+\lambda \sum_{\{u,v\}:u,v\in S}d(u,v)$\\
\hspace*{2cm} subject to $|S| = p$,\\\\
where $p$ is a given constant number and $\lambda$ is a parameter specifying a trade-off between the distance and submodular function.
We give a $2\alpha-$approximation for this problem.

Firstly we introduce our notations following \cite{borodin}. 
For any \mbox{$S \subseteq U$}, we let ${d(S)=\sum_{\{u,v\}:u,v\in S}d(u,v)}$.
We can also define \mbox{$d(S,T)$}, for any two disjoint sets $S$ and $T$ as: 
\begin{equation}
{d(S \cup T) - d(S) - d(T)}.\nonumber
\end{equation}
Let \mbox{$\phi(S)$} and $u$ be the value of the objective function and an element in \mbox{$U-S$} respectively. We can define the marginal gain of the distance function as 
\begin{equation}
\mbox{${d_u(S) = \sum_{v\in S}d(u, v)}$}\nonumber 
\end{equation}
and similarly marginal gain of the wight function as:
\begin{equation}
\mbox{$f_u(S) = f (S + u) - f (S)$}.\nonumber 
\end{equation}
The total marginal gain can also be defined using \mbox{$d_u(S)$} and  \mbox{$f_u(S)$} as 
\begin{equation}
\mbox{$\phi_u(S) = f_u (S) + \lambda d_u(S)$}.\nonumber 
\end{equation}
Let 
\begin{gather}
f'_u(S) = \frac{1}{2}f_u(S), \nonumber \\
\phi_u'(S) = f'_u (S) + \lambda d_u(S) \nonumber.
\end{gather}
Starting with an empty set $S$, the greedy algorithm (Algorithm~\ref{alg1}) adds an element $u$ from \mbox{$U - S$} in each iteration, in such a way that maximize \mbox{$\phi'_u(S)$}.
\begin{algorithm}                      
\caption{Greedy algorithm}          
\label{alg1}                           
\begin{algorithmic}[1]                    
\STATE \textit{Input}
      \STATE \hspace{\algorithmicindent} $U$: set of ground elements
      \STATE \hspace{\algorithmicindent} $p$: size of final set
\STATE \textit{Output}
      \STATE \hspace{\algorithmicindent} $S$: set of selected elements with size $p$
     \STATE \mbox{$S = \emptyset$}
\WHILE{\mbox{$|S| < p $}}
	\STATE \textit {find \mbox{$u \in U \setminus S$} maximizing $\phi'_u(S)$}
	\STATE \mbox{$S=S \cup \{u\}$}
\ENDWHILE
\RETURN \mbox{$S$}
\end{algorithmic}
\end{algorithm}
\begin{lemma}
\label{lemma:alpha}
Given an $\alpha$-relaxed triangle inequality semi-metric distance function $d(.,.)$, and two disjoint sets $X$ and $Y$, we have the following inequality:\\
$$\alpha(|X|-1)d(X,Y) \geq |Y|d(X)$$
\end{lemma}

\begin{proof}
Consider $u,v \in X$ and an arbitrary $w \in Y$. We know that:
$$\alpha(d(v,w) + d(w,u)) \geq d(u,v)$$
By changing $w$ we get:
$$\alpha(d(\{v\},Y)+d(\{u\},Y)) \geq |Y| d(u,v)$$ 
and then all combinations of $u$ and $v$:
$$\alpha (|X|-1) d(X,Y) \geq |Y| d(X)$$ 
\end{proof}

\begin{theorem}
Algorithm~\ref{alg1} achieves a $2\alpha$-approximation for solving Problem 1 with $\alpha$-relaxed distance $d(.,.)$ and monotone submodular function $f$.
\end{theorem}
\begin{proof}
Let $G_i$ be the greedy solution at the end of step $i$, $i < p$ and $G$ be the greedy solution at the end of the algorithm.
Suppose that $O$ is the optimal solution and let $A=O \cap G_i$, $B=G_i\setminus A$ and $C=O \setminus A$.
Obviously the algorithm achieves the optimal solution when $p=1$; thus we assume \mbox{$p>1$}. Now we consider two different cases: $|C| = 1$ and $|C| > 1$.
If $|C| = 1$ then $i = p-1$. Let $C = \{v\}$ and $u$ be the element that algorithm will take for the next (last) step. Then for all $v \in U\setminus S$ we have: 
\begin{align*}
\phi'_u(G_i) &\geq \phi'_v(G_i)\\
f'_u(G_i)+ \lambda d_u(G_i) &\geq f'_v(G_i)+ \lambda d_v(G_i)
\end{align*}
thus:
\begin{align*}
\phi_u(G_i) &= f_u(G_i)+ \lambda d_u(G_i) \\
            &\geq f'_u(G_i)+ \lambda d_u(G_i)\\
            &\geq f'_v(G_i)+ \lambda d_v(G_i)\\
            &\geq \frac{1}{2}\phi_v(G_i)
\end{align*}
as a result $\phi(G) \geq \frac{1}{2}\phi(O) \geq  \frac{1}{2\alpha}\phi(O)$.\\
Now consider $|C| > 1$.
By using Lemma~\ref{lemma:alpha} we have the following inequalities:\\
\begin{align}
\label{eq1}
\alpha (|C|-1)d(B,C) \geq |B|d(C)\\
\label{eq2}
\alpha (|C|-1)d(A,C) \geq |A|d(C)\\
\label{eq3}
\alpha (|A|-1)d(A,C) \geq |C|d(A)
\end{align}
$A$ and $C$ are two disjoint sets and we know that $A \cup C = O$; thus:
\begin{align}
\label{eq4}
d(A,C) + d(A) + d(C) = d(O)
\end{align}
We can assume that $p > 1$ and $|C| > 1$ (The greedy algorithm obviously finds the optimal solution when $p=1$).
Then following multipliers are applied to equations \ref{eq1}, \ref{eq2}, \ref{eq3}, \ref{eq4} respectively:\\
\begin{center}
$\frac{1}{(|C|-1)}$, $\frac{|C|-|B|}{p(|C|-1)}$, $\frac{i}{p(p-1)}$, $\frac{i|C|}{\alpha p(p-1)}$. 
\end{center}
If we add them, we have:
\begin{align*}
d(B,C) + d(A,C)
- d(A,C)\frac{i|C|(1-\frac{1}{\alpha})}{p(p-1)}
- d(C)\frac{i|C|(p-|C|)}{\alpha p(p-1)(|C|-1)}
\geq d(O) \frac{i|C|}{\alpha p(p-1)} \nonumber
\end{align*}
Since $p > |C|$ and $\alpha \geq 1$,
$$d(A,C)+d(B,C) \geq d(O) \frac{i|C|}{\alpha p(p-1)}.$$
thus (we substituted $\frac{1}{\alpha}$ with $x$, thus $0<x \leq 1$),
$$d(C,G_i) \geq d(O) \frac{xi|C|}{p(p-1)}$$
From the submodularity of $f'(.)$ we can get
$$\sum_{v \in C}f'_v(G_i) \geq f'(C \cup G_i) - f'(G_i)$$
also the monotonity of $f'(.)$ suggests that
$$f'(C \cup G_i) - f'(G_i) \geq f'(O) - f'(G).$$
Subsequently we have:
$$\sum_{v \in C}f'_v(G_i) \geq f'(O) - f'(G).$$
Therefore
\begin{align*}
\sum_{v \in C}\phi'_v(G_i) &= \sum_{v \in C}[f'_v(G_i)+ \lambda d(\{v\},G_i)]\\
                           &= \sum_{v \in C}f'_v(G_i)+ \lambda d(C,G_i)\\
                           &\geq [f'(O) - f'(G)] + d(O) \frac{\lambda xi|C|}{p(p-1)}.
\end{align*}
Let $u_{i+1}$ be the element taken at step $(i + 1)$, then we have
$$\phi'_{u_{i+1}}(G_i) \geq \frac{1}{p} [f'(O) - f'(G)] + d(O) \frac{\lambda xi}{p(p-1)}.$$
If we sum over all $i$ from 0 to $p-1$, we have
$$\phi'(G)=\sum_{i=0}^{p-1} \phi'_{u_{i+1}}(G_i) \geq [f'(O) - f'(G)] + d(O) \frac{\lambda x}{2}$$
Hence,
$$f'(G)+\lambda d(G) \geq f'(O) - f'(G) + d(O) \frac{\lambda x}{2}$$
and
\begin{align*}
\phi(G)=f(G)+\lambda d(G) &\geq \frac{1}{2}[f(O)+x\lambda d(O)] \\
					      &\geq \frac{x}{2}[f(O)+\lambda d(O)]\\
						  &= \frac{1}{2\alpha}\phi(O).
\end{align*}
\qed
\end{proof}

\section*{Problem 2. Max-Sum Diversification for Matroids}
Let $U$ be the underlying ground set, and $\mathcal{F}$ be the set of independent subsets of $U$ such that $\mathcal{M} = <U, \mathcal{F}>$ is a matroid.
Let $d(.,.)$ be a semi-metric distance function on $U$ and $f(.)$ be a non-negative monotone submodular set function measuring the weight of the subsets of $U$.
This problem aims to find a subset $S \subseteq \mathcal{F}$ that:\\\\
\hspace*{2cm} maximizes $f(S)+\lambda \sum_{\{u,v\}:u,v\in S}d(u,v)$\\\\
where $\lambda$ is a parameter specifying a trade-off between the two objectives.
Again, $\phi(S)$ is the value of the objective function. Because of the monotonicity of the $\phi(.)$, $S$ should be a basis of the matroid $\mathcal{M}$.
We give a $2\alpha-$approximation for this problem.

Without loss of generality, we assume that the rank of the matroid is greater than one. Let

\begin{displaymath}
\{x,y\} = \argmax_{x,y \in \mathcal{F}}[f(\{x,y\})+\lambda d(x,y)].
\end{displaymath}
We now consider the following local search algorithm:
\begin{algorithm}                      
\caption{Local Search algorithm}          
\label{alg2}                           
\begin{algorithmic}[1]                    
\STATE \textit{Input}
      \STATE \hspace{\algorithmicindent} $U$: set of ground elements
      \STATE \hspace{\algorithmicindent} $\mathcal{M=<U,\mathcal{F}>}$: a matroid on $U$
      \STATE \hspace{\algorithmicindent} $S$: a basis of $\mathcal{M}$ containing both $x$ and $y$
\STATE \textit{Output}
      \STATE \hspace{\algorithmicindent} $S$
\WHILE{$\exists \{u \in (U-S) \land v \in S \}$ such that $S+u-v \in \mathcal{F} \land \phi(S+u-v) > \phi(S)$}
	\STATE \mbox{$S=S + u - v$}
\ENDWHILE
\RETURN \mbox{$S$}
\end{algorithmic}
\end{algorithm}
\begin{theorem}
\label{matroid_theorem}
Algorithm~\ref{alg2} achieves an approximation ratio of $2\alpha ^2$ for max-sum diversification with a matroid constraint. 
\end{theorem}

As the algorithm is optimal for the case that the rank of the matroid is two, we assume that the rank of the matroid is greater than two.
The notation is like before and $O$ and $S$ are the optimal solution and the solution at the end of the local search algorithm, respectively. Let $A = O \cap S$, $B = S - A$ and $C = O - A$.
We utilize the following two lemmas from the \cite{borodin}.

\begin{lemma}
\label{bijective}
For any two sets $X, Y \in \mathcal{F}$ with $|X| = |Y|$, there is a bijective mapping $g : X \rightarrow Y$ such that  $X - x + g(x) \in \mathcal{F}$ for any $x \in X.$
\end{lemma}
Since both $S$ and $O$ are bases of the matroid, they have the same cardinality; subsequently, $B$ and $C$ have the same cardinality, too. 
Let $g : B \rightarrow C$ be the bijective mapping results from Lemma~\ref{bijective} such that $S - b + g(b) \in F$ for any $b \in B$. Let $B = \{b_1 , b_2 , . . . , b_t\}$, and let $c_i = g(b_i)$ for all $i$.
As claimed before, since the algorithm is optimal for $t=1$, we assume $t \geq 2$.
\begin{lemma}
\label{func_lemma}
$\sum_{i=1}^t f(S-b_i+c_i) \geq (t-2)f(S)+f(O).$
\end{lemma}

Now we are going to prove two lemmas regarding to our semi-metric distance function.
\begin{lemma}
\label{dist_temp_lemma}
If $t > 2$, $\alpha (d(B,C) - \sum_{i=1}^t d(b_i, c_i)) \geq d(C).$
\end{lemma}

\begin{proof}
For any $b_i, c_j, c_k,$ we have
$$\alpha (d(b_i,c_j)+d(b_i,c_k)) \geq d(c_j,c_k).$$
Summing up these inequalities over all $i, j, k$ with $i \neq j$, $i \neq k$, $j \neq k$, we have each $d(b_i , c_j)$ with $i \neq j$ is counted $(t - 2)$ times; and each $d(c_i , c_j)$ with $i \neq j$ is counted $(t - 2)$ times. Therefore
$$\alpha (t-2)[d(B,C) - \sum_{i=1}^t d(b_i, c_i)] \geq (t-2)d(C),$$
and the lemma follows.
\end{proof}

\begin{lemma}
\label{dist_lemma}
$\sum_{i=1}^t d(S-b_i+c_i) \geq (t-2)d(S) + \frac{1}{\alpha}d(O).$
\end{lemma}
\begin{proof}
\begin{align*}
\sum_{i=1}^t d(S-b_i+c_i)& \\
&= \sum_{i=1}^t[d(S) + d(c_i, S-b_i)-d(b_i, S-b_i)]\\
&= td(S) + \sum_{i=1}^t d(c_i, S-b_i) - \sum_{i=1}^t d(b_i, S-b_i)\\ 
&= td(S) + \sum_{i=1}^t d(c_i, S) - \sum_{i=1}^t d(c_i, b_i) - \sum_{i=1}^t d(b_i, S-b_i)\\ 
&= td(S) +  d(C, S) - \sum_{i=1}^t d(c_i, b_i) - d(A, B) - 2d(B). 
\end{align*}
There are two cases. If $t > 2$ then by Lemma \ref{dist_temp_lemma} we have 
\begin{align*}
d(C, S) - \sum_{i=1}^t d(c_i, b_i)\\
&= d(A, C) + d(B, C) - \sum_{i=1}^t d(c_i, b_i)\\
&\geq d(A, C) + \frac{1}{\alpha}d(C).
\end{align*}
We know that
\begin{align*}
d(S) = d(A) + d(B) + d(A, B)
\end{align*}
thus we have 
\begin{align*}
2d(S) - d(A, B) - 2d(B) \geq d(A). 
\end{align*}
Therefore
\begin{align*}
\sum_{i=1}^t d(S-b_i+c_i)& \\
&= td(S) + d(C, S) - \sum_{i=1}^t d(c_i, b_i) - d(A, B) - 2d(B)\\
&\geq (t-2)d(S) + d(A, C) + \frac{1}{\alpha}d(C) + d(A)\\
&\geq (t-2)d(S) + \frac{1}{\alpha}d(O)
\end{align*}
if $t=2$, then since the rank of the matroid is greater than two, $A \neq \emptyset$. Let $z$ be an element in $A$, then we have

\begin{align*}
&2d(S) +  d(C, S) - \sum_{i=1}^t d(c_i, b_i) - d(A, B) - 2d(B)\\
&= d(A, C) + d(B, C) - \sum_{i=1}^t d(c_i, b_i) + 2d(A) + d(A, B)\\
&\geq d(A, C) + d(c_1, b_2) + d(c_2, b_1) + d(A) + d(z, b_1) + d(z, b_2)\\
&\geq d(A, C) + d(A) + \frac{1}{\alpha}d(c_1, z) + \frac{1}{\alpha}d(c_2, z)\\
&\geq d(A, C) + d(A) + \frac{1}{\alpha ^2}d(c_1, c_2)\\
&\geq \frac{1}{\alpha ^2}(d(A, C) + d(A) + d(C))\\
&\geq \frac{1}{\alpha ^2}d(O)\\
\end{align*}
Therefore
\begin{align*}
&\sum_{i=1}^t d(S-b_i+c_i) \\
&= td(S) +  d(C, S) - \sum_{i=1}^t d(c_i, b_i) - d(A, B) - 2d(B)\\
&\geq (t-2)d(S) + \frac{1}{\alpha ^2}d(O).
\end{align*}
This completes the proof.
\end{proof}

Now we can complete the proof of Theorem~\ref{matroid_theorem}.
\begin{proof}
Since $S$ is a locally optimal solution, we have $\phi(S) \geq \phi(S-b_i+c_i)$ for all $i$. Therefore for all $i$ we have 
$$f(S) + \lambda d(S) \geq f(S-b_i+c_i) + \lambda d(S-b_i+c_i)$$
Summing up over all $i$, we have
$$tf(S)+\lambda td(S) \geq \sum_{i=1}^tf(S-b_i+c_i) + \lambda \sum_{i=1}^t d(S-b_i+c_i)$$
By Lemma~\ref{func_lemma} we know
$$tf(S)+\lambda td(S) \geq (t-2)f(S) + f(O) + \lambda \sum_{i=1}^t d(S-b_i+c_i)$$
Then by Lemma~\ref{dist_lemma} we have
$$tf(S)+\lambda td(S) \geq (t-2)f(S) + f(O) + \lambda(t-2)d(S)+\frac{\lambda}{\alpha ^2}d(O)$$
Therefore,
$$2f(S)+2\lambda d(S) \geq f(O) +\frac{\lambda}{\alpha ^2}d(O)$$
Since $\alpha \geq 1$,
$$2f(S)+2\lambda d(S) \geq f(O) +\frac{\lambda}{\alpha ^2}d(O) \geq \frac{1}{\alpha ^2}\phi(O)$$
$$\phi(S) \geq \frac{1}{2\alpha ^2}\phi(O).$$
\end{proof}
\qed

\section*{Conclusion}
In this study we answer a proposed question in \cite{borodin} about the existence of a bound on max-sum diversification problem with semi-metric distances and give a $2\alpha$-approximation for this question in the case that there is not any matroid constraint and a 2$\alpha^2$-approximation for an arbitrary matroid constraint.
One interesting question that may be posed is whether it is possible to prove similar results for a non-monotone submodular function? 

\bibliographystyle{abbrv}
\bibliography{main.bib}

\begin{thebibliography}{1}

\bibitem{borodin}
A.~Borodin, H.~C. Lee, and Y.~Ye.
\newblock Max-sum diversification, monotone submodular functions and dynamic
  updates.
\newblock In {\em Proceedings of the 31st symposium on Principles of Database
  Systems}, pages 155--166. ACM, 2012.

\bibitem{chekuri1997web}
C.~Chekuri, M.~H. Goldwasser, P.~Raghavan, and E.~Upfal.
\newblock Web search using automatic classification.
\newblock In {\em Proceedings of the Sixth International Conference on the
  World Wide Web}, 1997.

\bibitem{chen2006less}
H.~Chen and D.~R. Karger.
\newblock Less is more: probabilistic models for retrieving fewer relevant
  documents.
\newblock In {\em Proceedings of the 29th annual international ACM SIGIR
  conference on Research and development in information retrieval}, pages
  429--436. ACM, 2006.

\bibitem{clarke2008novelty}
C.~L. Clarke, M.~Kolla, G.~V. Cormack, O.~Vechtomova, A.~Ashkan,
  S.~B{\"u}ttcher, and I.~MacKinnon.
\newblock Novelty and diversity in information retrieval evaluation.
\newblock In {\em Proceedings of the 31st annual international ACM SIGIR
  conference on Research and development in information retrieval}, pages
  659--666. ACM, 2008.

\bibitem{fagin2003comparing}
R.~Fagin, R.~Kumar, and D.~Sivakumar.
\newblock Comparing top k lists.
\newblock {\em SIAM Journal on Discrete Mathematics}, 17(1):134--160, 2003.

\bibitem{fagin1998relaxing}
R.~Fagin and L.~Stockmeyer.
\newblock Relaxing the triangle inequality in pattern matching.
\newblock {\em International Journal of Computer Vision}, 30(3):219--231, 1998.

\bibitem{o2002streaming}
L.~O'callaghan, A.~Meyerson, R.~Motwani, N.~Mishra, and S.~Guha.
\newblock Streaming-data algorithms for high-quality clustering.
\newblock In {\em icde}, page 0685. IEEE, 2002.

\bibitem{veltkamp2001shape}
R.~C. Veltkamp.
\newblock Shape matching: Similarity measures and algorithms.
\newblock In {\em Shape Modeling and Applications, SMI 2001 International
  Conference on.}, pages 188--197. IEEE, 2001.

\bibitem{wang2007learn}
X.~Wang and C.~Zhai.
\newblock Learn from web search logs to organize search results.
\newblock In {\em Proceedings of the 30th annual international ACM SIGIR
  conference on Research and development in information retrieval}, pages
  87--94. ACM, 2007.

\end{thebibliography}

\end{document}